\documentclass{article}
\usepackage{ijcai23}

\usepackage{times}
\usepackage{soul}
\usepackage{url}
\usepackage[hidelinks]{hyperref}
\usepackage[utf8]{inputenc}
\usepackage[small]{caption}
\usepackage{graphicx}
\usepackage{amsmath}
\usepackage{amsthm}
\usepackage{booktabs}
\usepackage{algorithm}
\usepackage{algorithmic}
\usepackage[switch]{lineno}
\usepackage{caption}
\usepackage{subcaption}
\usepackage{multirow}

\DeclareMathOperator{\logme}{logME}
\DeclareMathOperator{\totlogme}{TLogME}
\DeclareMathOperator{\transrate}{TrR}
\DeclareMathOperator{\pooling}{Pooling}
\DeclareMathOperator{\avgpool}{AveragePooling}
\DeclareMathOperator{\roialign}{ROI-Align}

\newtheorem{proposition}{Proposition}[section]

\title{Transferability Metrics for Object Detection}

\author{
Louis Fouquet$^1$
\and
Simona Maggio$^1$\And
Léo Dreyfus-Schmidt$^{1}$
\affiliations
$^1$Dataiku Lab\\
\emails
\{name, surname\}@dataiku.com
}

\begin{document}

\maketitle

\begin{abstract}
    Transfer learning aims to make the most of existing pre-trained models to achieve better performance on a new task in limited data scenarios. 
    However, it is unclear which models will perform best on which task, and it is prohibitively expensive to try all possible combinations.
    If transferability estimation offers a computation-efficient approach to evaluate the generalisation ability of models, prior works focused exclusively on classification settings.
    To overcome this limitation, we extend transferability metrics to object detection. 
    We design a simple method to extract local features corresponding to each object within an image using ROI-Align.  
    We also introduce $\totlogme$, a transferability metric taking into account the coordinates regression task.
    In our experiments, we compare $\totlogme$ to state-of-the-art metrics in the estimation of transfer performance 
    of the Faster-RCNN object detector. We evaluate all metrics on source and target selection tasks, 
    for real and synthetic datasets, and with different backbone architectures.
    We show that, over different tasks, $\totlogme$ using the local extraction method provides a 
    robust correlation with transfer performance and outperforms other transferability metrics on local and global level features.
\end{abstract}

\section{Introduction}\label{intro}

In real-world industrial vision tasks, labeled data is scarce, and training models from scratch can be excessively costly and time-consuming.
Thus, transfer learning is commonly used to take advantage of pre-trained models.
As several pre-trained models and many large datasets for pre-training are available, 
Transferability Metrics (TM) have been designed to estimate the model's ability to generalize to the end task
and proceed to model selection with limited computational capacities by avoiding directly transferring all the models.

Existing works for classification tasks present different transferability metrics based on scores between features extracted from pre-trained models and labels of a target dataset. They can be used either to select the best
source or target dataset (source and target selection) in a transfer learning setting, but also to select the best architecture within a model zoo (model selection) or even the optimal number of layers to transfer (layer selection).
These metrics are free of optimization and only need access to a pre-trained model to compute pairs of feature-labels.
However, for Object Detection (OD) tasks, each image may contain multiple objects with both regression and classification tasks. Feature-label pairs cannot thus be easily defined and no work has yet been done on estimating  
transferability in this setting.

To solve this problem, we design a method to extract local features associated with each object using the ROI-Align pooling proposed in the Mask-RCNN architecture \cite{he2017maskrcnn}.
We then adapt state-of-the-art transferability metrics by generating again local-features-label pairs for each object as in an image classification setting.
In addition, we propose a novel version of $\logme$ \cite{you2021logme} to take into account both the classification and the regression tasks, $\totlogme$.

We tested our approach on both synthetic datasets and real-life datasets.
We investigated especially the source and target selection tasks with a traditional Faster-RCNN \cite{ren2015fasterrcnn} architecture with a ResNet \cite{he2016resnet} backbone and a more recent architecture \cite{li2021benchmarking} adapting a Faster-RCNN to a Vision Transformer (ViT) backbone.

We summarize our contributions as follows:

\begin{figure}[t]
    \includegraphics[width=\columnwidth]{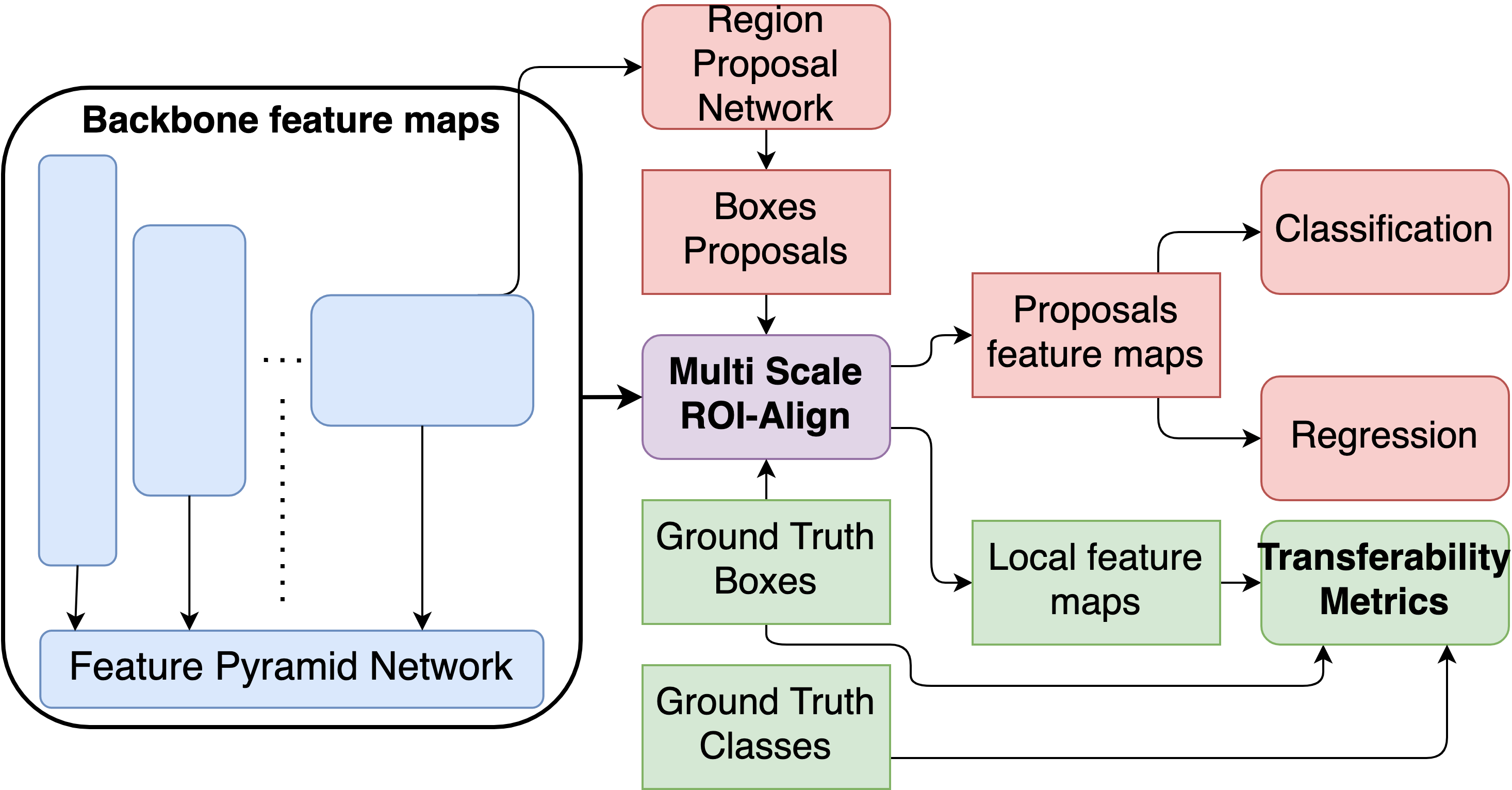}
    \caption{Overview of our method for extracting local features. The blocks of the Faster-RCNN architecture in red are bypassed, 
    while the flow in green enables the computation of transferability metrics on object-level features.}
    \label{fig:flow}
\end{figure}

\begin{enumerate}
    \item We propose to use efficient transferability metrics in an object-detection setting to evaluate model generalization.
    
    \item We introduce the $\totlogme$ transferability metric, an extension of $\logme$~\cite{you2021logme} 
    on both labels and bounding box coordinates, 
    and show it is upper bounded by the average log-likelihood of the optimal model.%
    
    \item We present a method to extract object-level features based on ROI-Align that enables using
    state-of-the-art transferability metrics for object detection and outperforms a simple baseline, 
    based on image-level representations from backbone layers.
    
    \item We analyze several transferability metrics on real and synthetic transfer scenarios and show empirically
    that $\totlogme$ with local feature extraction is a robust baseline
    for different tasks, especially with a constrained training budget.
\end{enumerate}

We provide in Section~\ref{sec:related_works} an overview of the prior works, whereas
Section~\ref{sec:tm_od} details the transferability metrics that we propose for OD,   
Section~\ref{sec:experiments} summarizes our empirical investigation, while Section~\ref{sec:conclusion} concludes the paper.

\section{Related Works}\label{sec:related_works}

\paragraph*{Transfer Learning}

The vast research topic of Transfer learning~\cite{thrun1998transferlearning} refers to transferring knowledge 
from a source task to a target task having in general a different output space. 

Here, we consider only inductive transfer, or fine-tuning~\cite{yosinski2014transferability}, which transfers a pre-trained model to a different target task, 
reusing or slightly adapting the pre-trained model as a feature extractor. 
Although there are various finetuning strategies~\cite{li2018inductive,gururangan2020pretraining,garg2020pretraining},
in this paper, we focus on a simple fine-tuning setting.

\paragraph*{Task Relatedness}
Empirical methods to assess task relatedness were historically developed. 
Taskonomy~\cite{zamir2018taskonomy} built a taxonomy of tasks by empirically computing transferability between different tasks.
Task2Vec~\cite{achille2019task2vec} proposed to use the Fisher Information Matrix of a probe network fine-tuned on the target task to embed the latter. 
However, these methods require model training and are thus computationally expensive.

\paragraph*{Datasets Similarity}
Task Similarity has also been investigated through distance metrics such as $\mathcal{A}$-distance~\cite{ben2006analysis}, maximum mean discrepancy~\cite{gretton2012kernel}. 
or Wasserstein distance between features of source and target~\cite{cui2018large}.
But these methods lacked asymmetry, which is an intrinsic property of transferability, and did not take advantage of the labeled nature of the data. 
In OTDD~\cite{alvarez2020geometric} the authors proposed to also compute a Wasserstein label-to-label distance.
And, in OTCE~\cite{tan2021otce}, the similarity is a combination of a Wasserstein distance on features and a conditional entropy between source and target labels 
with their joint distribution modeled through optimal coupling.
Despite their simplicity, these methods need access to the source data used to pre-train the model, which is rarely available in real-world applications.
In RSA~\cite{dwivedi2019representation} and DEPARA~\cite{song2020depara}, only access to the pre-trained model is necessary and the authors proposed to use representation similarity 
analysis and graph similarity respectively to measure transferability.
They are unfortunately computationally comparable to fine-tuning on the target dataset~\cite{huang2022frustratingly} which makes them hard to use in practice.

\paragraph*{Transferability Metrics}
Other analytical metrics~\cite{tran2019transferability,nguyen2020leep,you2021logme,bao2019information,huang2022frustratingly} are more aligned with our goal of proposing a computationally fast method that does not need retraining nor access to source data.
NCE~\cite{tran2019transferability} proposed to compute conditional entropy between soft labels (associated with the source task) and true labels of the target dataset. 
LEEP~\cite{nguyen2020leep} replaced the conditional entropy with an empirical predictor between soft and true labels
and its extension $\mathcal{N}$LEEP~\cite{li2021ranking} proposed to better estimate soft labels using Gaussian Mixture Models on PCA-reduced features at an extra computational cost.
H-score~\cite{bao2019information} used an information theoretic approach to estimate transferability but suffered from instability for high dimensional data which was corrected 
using Ledoit-Wolf estimator~\cite{ibrahim2021newer}.
In~\cite{you2021logme}, authors proposed to compute the maximum evidence of a linear model, using a Gaussian prior on weights, between features extracted from the pre-trained models and 
associated target labels.
We highlight that this last method is the only one designed for regression and not classification.
A more recent work proposed TransRate~\cite{huang2022frustratingly} that computes mutual information between extracted features and target labels. \\
Despite having the desirable properties that we aim for, these metrics cannot be used directly in an OD setting where images have more than one label 
and that combines both classification and regression tasks.

\paragraph*{Object Detection}

After pioneer work on object detectors,
two types of deep learning detectors have been introduced: two-stage detectors (such as Faster-RCNN~\cite{ren2015fasterrcnn}),
and single-stage detectors such as YOLO~\cite{redmon2016yolo,bochkovskiy2020yolov4}).

In this paper, we focus on Faster-RCNN and its improvements, as this architecture exhibits high accuracy on small objects, and is the most popular choice in industrial contexts, 
where real-time speed is not needed.
An important improvement of Faster-RCNN is the Feature Pyramid Network (FPN) [45], building a feature pyramid
at multiple levels to improve the detection of small objects.
Mask R-CNN~\cite{he2017maskrcnn} extends on the Faster R-CNN by adding a segmentation branch and introducing the RoIAlign layer, 
to avoid localization misalignment due to spatial quantization.

Finally, very recently the first attempts at replacing CNNs with Transformers appeared for both two-stage detectors, 
with Faster-RCNN-ViT~\cite{li2021benchmarking}, 
and single-stage detectors, with the Swin Transformer~\cite{liu2021swin}, 
achieving state-of-the-art performances, but using a higher number of parameters.

More improvements and new models have been proposed, especially for resource-constrained environments: we refer to~\cite{zaidi2022survey} for an in-depth survey of object detectors.

\section{Transferability Metrics for Object Detection}\label{sec:tm_od}

\subsubsection{Notations}
Let $\mathcal{D}=\{(\mathbf{x}, \mathbf{y}, \mathbf{t})\}$ be a dataset where $\mathbf{t}$, $y$ and $\mathbf{x}$ 
are an object's location, its class, and the image where it is present.
A model $\mathcal{M}$, pre-trained on a source task $\mathcal{T}_S$, is transferred to the target task $\mathcal{T}_T$,
learning the model $\mathcal{M}_T$ on $\mathcal{D}$.
\newline
\newline
\begin{tabular}{cp{0.6\textwidth}}
    $F$ & feature matrix of target samples \\
    & extracted from $\mathcal{M}$\\
    $\phi_l$ & global (image-level) feature extractor at the $l$-th \\
    & layer of Faster-RCNN backbone \\
    $\varphi_l$ & local (object-level) feature extractor at the $l$-th \\
    & layer of Faster-RCNN head \\
    $f_l$ & feature extractor at the object level \\
\end{tabular}\\

\subsection{Problem Setting}

Let $\mathcal{D}=\{(\mathbf{x}_i, y_i, \mathbf{t}_i)\}_{i=1}^n$ be the target dataset of the transfer, 
containing overall $n$ objects, 
each represented by $\mathbf{x}_i$, the image containing it.
OD is thus a multi-task prediction problem, 
aiming at estimating for any object $i$ its position as bounding boxes coordinates 
$\mathbf{t}_i=\{t_x, t_y, t_w, t_h\}_i$, and its class $y_i$ among $K$ different classes of objects.

The standard evaluation metric for OD models is the $\mathrm{mAP}@\mathrm{IOU}_{0.05:0.95}$~\cite{lin2014microsoft}, defined as the area under the 
precision-recall curve averaged across classes and various levels of Intersection Over Union (IOU)\footnote{\url{https://cocodataset.org/\#detection-eval}}. 
In the following, we call this metric $\mathrm{mAP}$ for simplicity. 

The goal of the Transferability Metric in OD is to approximate the ground-truth transfer performance $\mathrm{mAP}(\mathcal{M}_T)$, 
estimating the compatibility of the features of the objects in the target dataset $\mathcal{D}$ extracted from the pre-trained model, 
$\{f_i\}_{i=1}^n$, and the multi-task ground-truth labels $\{y_i\}_{i=1}^n$ and $\{\mathbf{t}_i\}_{i=1}^n$.
The underlying assumption is that the compatibility of pre-trained features and target labels serves as a 
a strong indicator for the performance of the transferred model.

\subsection{Performance of Transferability Metrics}

For $M$ transfer scenarios 
$\{\mathcal{T}_{Sm} \rightarrow \mathcal{T}_{Tm}\}_{m=1}^M$  
an ideal TM would produce scores $\{s_m\}_{m=1}^M$ ranking perfectly the transfers ground-truth performances 
$\{\mathrm{mAP}(\mathcal{M}_T)_m\}_{m=1}^M$.
Therefore, following previous works~\cite{you2021logme,huang2022frustratingly}, we use Pearson's linear correlation, Spearman's and Kendall's rank correlation between  
$\{\mathrm{mAP}(\mathcal{M}_T)_m\}_{m=1}^M$ and $\{s_m\}_{m=1}^M$ to assess the quality of TMs.
In the experimental results (Section \ref{sec:experiments}), we display Pearson's correlation for brevity but all metrics are available in the Supplementary Materials (Section ~\ref{sec:supp_mat}).

\subsection{Object-Level Feature Extraction}\label{subsec:feature_extraction}

As stated before, we focus here on the Faster-RCNN \cite{ren2015fasterrcnn}, 
both in its original form based on ResNet backbone and in its latest variant based on ViT \cite{li2021benchmarking}.
We consider a transfer protocol, where the head only is retrained, providing the finetuned model $\mathcal{M}_T$.

Let $f_l$ be a \textit{local} feature extractor, defined by the operations leading to the feature representation of an object 
at layer $l$ of the pre-trained model $\mathcal{M}$. Depending on the position of $l$ in the network, these operations 
include either convolution at the image level followed by pooling at the object level, or convolutions and dense 
operations directly at the object-level.

In OD, the extraction of features to use in TM is not trivial as in classification, as local features 
related to a ground truth bounding box cannot be retrieved by a simple feedforward run of the network.
We propose several local feature extraction options and show empirically that the choice of the method 
is critical for the success of TM.

The Faster-RCNN model $\mathcal{M}$ is composed of $L$ backbone layer blocks, 
$\phi_L \circ \cdots \circ \phi_{1}$.
The $L$ backbone layers are also connected to a Feature Pyramid Network (FPN)\cite{lin2017feature}, 
with layer blocks $\phi_{\mathrm{FPN}, 1}$ to $\phi_{\mathrm{FPN}, L}$. 

These backbone layers are connected to the head, containing a Region Proposal Network (RPN) 
with the role of detecting objects (i.e. passing bounding box proposals). In the context of TM the RPN needs to be bypassed 
as we want to compute features for already known ground truth object locations, as illustrated in Figure \ref{fig:flow}. 

Given an object bounding box $\mathbf{t}_i$ and a global feature map $\phi_l(\mathbf{x}_i)$ at layer $l$, 
a pooling operation extracts an object-level feature map $
f_l(\mathbf{x}_i, \mathbf{t}_i) = \pooling(\phi_l(\mathbf{x}_i), \mathbf{t}_i)$, 
that we can use as a local representation. 

Finally, the network applies convolutional and 
dense layers $\varphi$ to the previous feature map at the last layer $L$ to produce the final object representations, 
$f_l(\mathbf{x}_i, \mathbf{t}_i) = \varphi(\pooling(\phi_L(\mathbf{x}_i), \mathbf{t}_i))$,
eventually processed by the end classifier and regressor. 

\subsubsection{Global features as a Baseline}

As a simple baseline, we can use a global feature map $\phi_l(\mathbf{x}_i)$ from layer $l$ as features for each object $i$ present in the image.
All the objects present in the image have identical features, given by the feature extractor $f^B_l$ which is the average pooling of $\phi_l(\mathbf{x}_i)$ over the spatial dimension:
\begin{align}
    f^B_l(\mathbf{x}_i)=\avgpool(\phi_l(\mathbf{x}_i))
\end{align}
Similarly, we define as $f^B_{\mathrm{FPN}\, l}$ the global feature extractor based on the $l$-th block of FPN.
For this baseline $f_l(\mathbf{x}_i, \mathbf{t}_i)=f_l(\mathbf{x}_i)$ as the object location is ignored.

\subsubsection{Pooling local features by ROI-Align}

Let $\phi_l$ be a global feature extractor from layer $l$, and 
$\mathbf{t}=(t_x, t_y, t_w, t_h)$ a bounding box of an object on the original image. 
The scale ratio $s$ is the ratio between the input image size and the feature map size.

The object-level feature extractors applying ROI-Align at scale $s$ and layer $l$ are defined as follows:
\begin{align}
    f^{\mathrm{ROI}}_l(\mathbf{x}_i, \mathbf{t}_i)=\roialign(\phi_l(\mathbf{x}_i), \mathbf{t}_i, s)
\end{align}
where the ROI-Align pooling operation results in a feature map of a reduced size given by the pooling dimension.
Similarly, the feature extractor based on the $l$-th block of FPN is defined as $f^{\mathrm{ROI}}_{\mathrm{FPN},l}(\mathbf{x}_i, \mathbf{t}_i)=\roialign(\phi_{\mathrm{FPN},l}(\mathbf{x}_i),\mathbf{t}_i, s)$.

\subsubsection{Pooling global features by Multi-Scale ROI-Align}

With feature maps from different layers $\phi_1,\cdots,\phi_L$ and thus different scale ratios $s_1, \cdots, s_L$ we can define 
the Multi-Scale ROI-Align~\cite{lin2017feature} object-level features, which intuitively select the most appropriate layer
depending on the size of the object: 
\begin{align}
    f^{\mathrm{MS}}(\mathbf{x}_i, \mathbf{t}_i) = \roialign(\phi_{\mathrm{FPN},k}(\mathbf{x}_i),\mathbf{t}_i, s_k)
\end{align}
with $k = \lfloor k_0 + log_2 (\sqrt{t_w\cdot t_h}/s_0) \rfloor$, where $k_0$ and $s_0$ are canonical layer and scale, 
ensuring the last layer is selected for an object of the size of the whole image. 

\subsubsection{Local features from dense layers}

Instead of using the local features obtained by pooling, we can forward the local representations
from the last backbone layer
through the final $G$ convolutional and dense layers of the network operating at the object level, 
$\varphi = \varphi_G \circ \cdots \circ \varphi_{1}$. This produces different local features depending on the layer: 
\begin{align}
    f^{FC}_l(\mathbf{x}_i, \mathbf{t}_i) = \varphi(\pooling(\phi_L(\mathbf{x}_i)), \mathbf{t}_i)
\end{align}
In particular, we call the feature from the penultimate dense layer $f^{FC}_{-1}$.

\subsection{Transferability Metrics for Classification}\label{subsec:background}

State-of-the-art TMs are defined for a classification task as compatibility measures between the feature matrix $F$ of the \textit{target samples}
extracted from the \textit{pre-trained model} and the labels of the target samples $y$. 
The features are usually extracted from the penultimate layer of the network but different schemes have been evaluated.

\subsubsection{LogME}

The logarithm of maximum evidence~\cite{you2021logme}, $\logme$, is defined as:
\begin{align}
    \logme(y, F) = \max\limits_{\alpha, \beta} \log p(y|F, \alpha, \beta)
\end{align}%
with $p(y|F, \alpha, \beta) = \int p(w|\alpha) p(y|F, \beta, w) dw$, 
where $\alpha$ and $\beta$ parametrize the prior distribution of weights $w\sim \mathcal{N}(0,\alpha^{-1}I)$, 
and the distribution of each observation $p(y_i|f_i,w,\beta)=\mathcal{N}(y_i|w^T f_i,\beta^{-1})$.
The $\logme$ has been designed for regression problems and is extended to $K$-class classification by considering each class 
as a binary variable on which the $\logme$ can be computed and finally averaging the $K$ scores. 

\subsubsection{H-Score}
The H-score~\cite{bao2019information} is defined as:
\begin{align}
    \mathcal{H}(y, F) = tr(\Sigma^{(F)^{-1}}\Sigma^{(z)})
\end{align}
where $\Sigma^{(F)}$ is the feature covariance matrix, 
and $\Sigma^{(z)}$ is the covariance matrix of the target-conditioned feature matrix $z=E[F|y]$.
However, as estimating covariance matrices can suffer from instability for high dimensional data, a shrinkage-based H-score has been proposed. 
This regularized H-score~\cite{ibrahim2021newer} is defined as: 
\begin{align}
    \mathcal{H}_{\alpha}(y, F) = tr(\Sigma^{(F)^{-1}}_{\alpha} (1 - \alpha)\Sigma^{(z)})
\end{align}
where $\Sigma^{(F)}_{\alpha}$ is 
the regularized covariance matrix of shrinkage parameter $\alpha$ computed using the Ledoit-Wolf estimator. 

\begin{table*}[h!]
    \centering
    \begin{tabular}{lrr|rr|rrrrr}
        \toprule
        {} &     & &  \multicolumn{2}{c|}{Global Features} & \multicolumn{5}{c}{Local Features}\\
        \midrule
        {} &     & & Best Metric & Score & LogME & TLogME & $\mathcal{H}$ & $\mathcal{H}_{\alpha}$ & $\transrate$  \\
         Datasets & Task    & Backbone &  &             &         &          &           \\
         \midrule
         Synthetic & Source & ResNet    & $\mathcal{H}$  &      0.42\textsuperscript{ } &    \textbf{0.50}\textsuperscript{ } &       \textbf{0.50} &    0.36\textsuperscript{ } &     0.49\textsuperscript{ } &      0.21\textsuperscript{ } \\
         Synthetic & Target & ResNet     & $\mathcal{H}_\alpha$  &      0.41\textsuperscript{ } &    0.32\textsuperscript{ } &        0.35 &    -0.22\textsuperscript{ } &     \textbf{0.42}\textsuperscript{ } &      0.12\textsuperscript{ } \\
         Real 1 & Target  & ResNet   & $\logme$    &      0.00\textsuperscript{*} &  0.10\textsuperscript{*} &       \textbf{0.47} &   -0.01\textsuperscript{*} &    -0.03\textsuperscript{*} &     -0.68\textsuperscript{ } \\
         Real 2  & Target   & ResNet  &$\totlogme$    &       0.33\textsuperscript{ } &     0.15\textsuperscript{*} &        0.31 &   -0.20\textsuperscript{ } &    -0.24\textsuperscript{ } &      \textbf{0.43}\textsuperscript{ } \\
         Real 2 & Target & ViT  & $\transrate$      &     0.54\textsuperscript{ } &  \textbf{0.56}\textsuperscript{ }  &       \textbf{0.56} &    0.04\textsuperscript{*} &     0.04\textsuperscript{*} &      0.50\textsuperscript{ } \\
         \bottomrule
        \end{tabular}
    \caption{Correlation between mAP and transferability metrics for different tasks. For synthetic datasets the correlation is the mean of correlations. Asterisks (*) indicate non significant p-values.}
    \label{tab:summary}
\end{table*}

\subsubsection{TransRate}

TransRate~\cite{huang2022frustratingly} is defined as: 
\begin{align}
    \transrate(y,F) = h(F) - h(F|y)\approx H(F^\Delta) - H(F^\Delta|y)
\end{align}
where $h$ is the differential entropy of a 
continuous random variable, $H$ is the entropy of a discrete random variable and $\Delta$ is the quantization error 
of the quantized features $F^\Delta$.

\subsection{Extension of Transferability Metrics}

Once defined the object-level feature maps, $f_l(\mathbf{x}_i, \mathbf{t}_i)$, with $f_l$ a feature extractor among $f^B_l$, $f^B_{\mathrm{FPN}, l}$, $f^{\mathrm{ROI}}_l$, $f^{\mathrm{ROI}}_{\mathrm{FPN}, l}$, 
$f^{\mathrm{MS}}$ or $f^{FC}_l$, 
we can plug them as the feature matrix $F$ in the scores defined in Subsection \ref{subsec:background}
to obtain TMs for the OD task.

For the particular case of $\logme$ we can compute an additional score evaluating the 
compatibility of features with the bounding box target $\mathbf{t}$. Averaging both the classification and the regression score, 
we obtain a total $\logme$ taking into account the multi-task nature of the OD problem.

Extending other transferability metrics to take into account the coordinate regression task would involve an arbitrary discretization 
of the bounding boxes' coordinates. In addition, both $\transrate$ and $\mathcal{H}$ have undesirable properties (i.e. being null) in pure detection cases
where there is only one class. We thus decided to focus in this work only on the extension of $\logme$.

\subsubsection{Total LogME}

For the compatibility of features with the position $\mathbf{t}=\{t_k\}_{k=1}^4$, we can average along the tasks to 
build the $\logme$ specific to the regression part:
\begin{equation}
    \logme_{pos}(\mathbf{t},F) = \frac{1}{4}\sum_{k=1}^4 \max\limits_{\alpha, \beta} \log p(t^k|F, \alpha, \beta)
\end{equation}
One-hot encoding $y$ and expanding in one dimension, $y^k$, per class, we have a $\logme$ specific for classification, 
averaging along each dimension:
\begin{equation}
    \logme_{class}(y,F) = \frac{1}{K}\sum_{k=1}^K \max\limits_{\alpha, \beta} \log p(y^k|F, \alpha, \beta)
\end{equation}
Eventually, we can average $\logme_{pos}$ and $\logme_{class}$ to have a unique score, taking into account both the 
prediction of the position and the class of objects:
\begin{equation}
    \totlogme(y,\mathbf{t}, F) = \logme_{pos} + \logme_{class}
\end{equation}

We empirically show that $\totlogme$ outperforms other TMs for the OD task.

\subsubsection{Theoretical Properties of Total LogME}\label{theory}

Similarly to previous work~\cite{huang2022frustratingly} we show that $\logme$ and $\totlogme$ provide a lower bound of the optimal log-likelihood,
which is closely related to the transfer performance, ensuring these metrics can be used as its proxy.

In particular, for tractability, we focus on the performance of the optimal transferred model $w^*$, when considering a pre-trained 
feature extractor that is not further finetuned~\cite{huang2022frustratingly}. 
This is close to our setting where we transfer 
the model head only.

Following~\cite{you2021logme} in a classification setting, we have that $\mathcal{L}(F, w^*)= \log p(y | F, w^*) $.
To extend this result to $\totlogme$, we consider the total likelihood $\mathcal{L}_{tot}(F, w^*) = \log p(y, \mathbf{t} | F, w^*)$.

\begin{proposition}
$\logme$ and $\totlogme$ provide a lower bound of the optimal log-likelihood: $\mathcal{L}(F, w^*) \geq \logme(y, F)$ 
and $\mathcal{L}_{tot}(F, w^*) \geq \totlogme(y, \mathbf{t}, F)$, assuming that object class $y$ and location $\mathbf{t}$ 
are independent conditionally on $w, F$.
\end{proposition}

\begin{proof}
We observe that: 
\begin{eqnarray}\label{eq:int_max}
   p(y | F, w^*) &\geq& p(y | F, w)   \nonumber \\
   \int p(w) dw \cdot p(y | F, w^*)&\geq& \int p(w) p(y|F,w) dw \nonumber \\
   p(y | F, w^*)&\geq& p(y|F).
\end{eqnarray}

Let us consider a linear head model $w$ and the graphical model in~\cite{you2021logme} 
parametrizing the linear head by $\alpha$ and the observations by $\beta$.
The parametrized likelihood is defined as 
$\mathcal{L}_{\alpha\beta}(F, w_{\alpha}, \beta) = \log p(y | F, w_{\alpha}, \beta)$.

The space of linear heads $w$ being larger than the parametrized one $w_\alpha$, 
we have the inequality 
$p(y | F, w^*) \geq p(y | F, w^*_{\alpha^*}, \beta^*)$, 
where $\alpha^*, \beta^* = \arg \max \limits_{\alpha, \beta} p(y | F, w_{\alpha}, \beta)$
and $w^*_{\alpha^*} = \arg \max \limits_{w_{\alpha^*}} p(y | F, w_{\alpha^*}, \beta^*)$.

Following the same reasoning as in (\ref{eq:int_max}), we have that 
$p(y | F, w^*_{\alpha^*}, \beta^*)\geq p(y|F, \alpha^*, \beta^*)$.

Thus $\mathcal{L}(F, w^*) \geq \mathcal{L}_{\alpha\beta}(F, w^*_{\alpha^*}, \beta^*) \geq \logme (y, F)$.

Assuming that object class $y$ and location $\mathbf{t}$ are independent conditionally on $w, F$,
\begin{eqnarray}
   \mathcal{L}_{tot}(F, w^*) &=& \log p(\mathbf{t} | F, w^*) + \log p(y | F, w^*) \\
   &\geq&   \logme_{pos}(\mathbf{t}, F) + \logme_{class}(y, F). \nonumber %
\end{eqnarray}

Thus $\mathcal{L}_{tot}(F, w^*) \geq \totlogme (y, \mathbf{t}, F)$.
\end{proof}

\section{Experiments}\label{sec:experiments}

As mentioned before, transferability metrics can be used for different tasks such as \textit{layer selection}, 
\textit{model selection}, and \textit{source and target selection}. Here we investigate:
\begin{enumerate}
    \item Source selection: Given pre-trained models on different source datasets the goal is to select the model that 
    will yield the best performance on the target dataset.
    \item Target selection: Given one pre-trained model and different target tasks the goal is to determine which 
    target task the model will transfer on with the best performance.
\end{enumerate}

For our experiments, we used both synthetic and real-life object-detection datasets and finetune solely the head of the Faster-RCNN (i.e backbone layers are frozen).
Everywhere in the following, except in subsection \ref*{subsec:layer}, local features are extracted from $f^{MS}$ and 
global features are extracted from the last layer of the FPN.
Overall we evaluate all TMs on $26$ tasks, including $5$ synthetic source selections, $5$ synthetic target selections, 
and $3$ real target selections, first applying TMs to global and then to local features.
For simplicity, in this section, we only show the best TM applied to the global features as a baseline,
named \textit{Best Global Metric} in the results, while showing all TMs for local features. Correlations marked with an asterisk (*) indicate non-significant p-values.

\subsection{Synthetic Datasets}
We first tested our method in a simple controlled setting by creating synthetic object-detection datasets\footnote{Adapting the code at \url{https://github.com/hukkelas/MNIST-ObjectDetection}} 
using the famous MNIST dataset~\cite{lecun1998gradient} and 
its related datasets FASHION MNIST~\cite{xiao2017fashion}, KMNIST~\cite{clanuwat2018deep}, EMNIST~\cite{cohen2017emnist} and USPS~\cite{hull1994database}.
We build $M=20$ transfer scenarios, considering each dataset in turn as the source $\mathcal{T}_S$ and transferring to all remaining datasets. 
This amounts to $5$ source selection tasks and $5$ target selection tasks. For each task, we compare the correlation between different 
TMs and performance for the different feature extraction methods presented in Section~\ref*{sec:tm_od}.

We train a Faster-RCNN with a ResNet backbone on each of the datasets for $30$ epochs to use as pre-trained models. Then for each pair of source-target datasets, 
we finetune the pre-trained model from the source task on the target task for only $5$ epochs to simulate transfer learning with a constrained training budget compared to the pre-trained model. 
In addition, as these synthetic tasks are relatively easy, convergence is obtained in a few epochs.
For each pair, we also extract object-level features of the target dataset using the pre-trained source model.

\begin{figure}[h]
    \includegraphics[width=0.95\columnwidth]{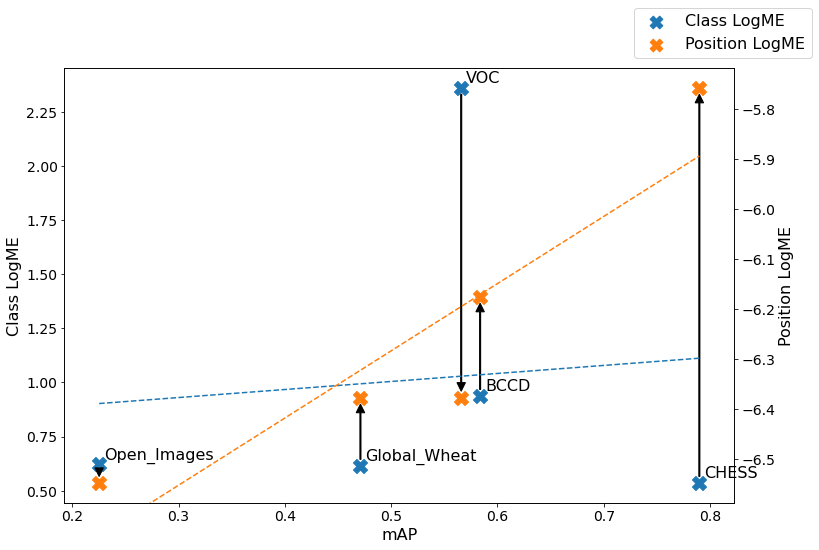}
    \caption{Comparison of $\logme_{class}$ and $\logme_{pos}$ on \textit{Real 1}.}
    \label{fig:posvsclass}
\end{figure}

\subsection{Real Datasets}
For real datasets, we only focus on the target selection task.
We build a first set (\textit{Real 1}) of target datasets with very different tasks: chess pieces, blood cells, wheat detection, VOC~\cite{Everingham10} 
and a subset of Open Images~\cite{kuznetsova2020open} with vehicles only. All datasets are publicly available. %
For each of these datasets, we use a Faster-RCNN pre-trained on COCO~\cite{lin2014microsoft} and transfer it to the target tasks, thus building one target selection task 
with $M=5$ transfer scenarios.
We also design a second set (\textit{Real 2}) of $20$ different target datasets by randomly sampling classes from the Open-Images dataset. To this end, we use the class co-occurrence matrix
to ensure the selection of images containing multiple object types. The co-occurence matrix counts the number of co-occurrences for any pair of classes, we filter out classes appearing in a small number of different images (see \ref*{subsec:reproducibility}). The first class is randomly selected while the four other classes are selected randomly but weighted by their co-occurrence count with the first class.
In this case, we thus build one target selection task with $M=20$ transfer scenarios.

We proceed as before by transferring the pre-trained Faster-RCNN to the different target tasks. For this experiment, we use both 
a ResNet backbone and a more recent Vision Transformer backbone, pre-trained on COCO and fine-tuned for $30$ epochs on $\mathcal{T}_T$. 
We also study the sensitivity of the TMs to the training budget, 
analyzing how the correlation between transferability metrics and performance evolves as the transfer proceeds from $5$ to $30$ epochs.

\subsection{Total LogME on local features efficiently measures transferability}

Table \ref{tab:summary} summarizes the correlation between mAP and transferability metrics for all the different tasks described above. 
Correlation results for synthetic data are averaged across the $5$ $\mathcal{T}_S$ for the target selection tasks and across the $5$ $\mathcal{T}_T$ 
for the source selection tasks.
We can observe in Table \ref{tab:summary} that computing TMs on local features is more effective than using global features for all tasks, especially on real datasets.

The large increase in overall correlation for all synthetic tasks when using local instead of global features is highlighted in Figure \ref{fig:correlation_agg},
where the correlation is computed on all $M=20$ transfer scenarios.
In addition, we observe that if regularized H-Score yields good results on synthetic datasets it fails on real-life tasks, while TransRate is inconsistent across these tasks. H-score exhibits an always near-zero and non-significant correlation, that could be tied to the instability of covariance estimation.
Over very different tasks, TLogME mitigates some of the limitations of LogME, especially when the regression task dominates the final performance. In Figure \ref{fig:posvsclass}, with traditional LogME, the CHESS dataset has a low transferability score but 
a high mAP. Indeed, the regression task is very easy as the pieces are disposed over a grid chess board. On the other side, VOC has very similar classes to the COCO dataset hence the high class-LogME. However, the complexity of the regression task makes it harder to have a high performance. %
Our proposed version of Total LogME outperforms the classical LogME in most cases for both local and global features ($18/26$ cases) as shown in the Supplementary Materials.
Overall, Total LogME on local features has a good correlation to mAP for real-life applications and a robust behavior overall, as illustrated in Figure \ref{fig:comparison_corr}.

\begin{table}[t]
    \centering
    \begin{tabular}{lrrrrrrrrrr}
        \toprule
        {} &       $\logme$ & $\totlogme$ & $\mathcal{H}$ & $\mathcal{H}_{\alpha}$ & $\transrate$ \\
        Epochs &                 &             &         &          &           \\
        \midrule
        5   &         0.60 &        0.59 &    0.38\textsuperscript{*} &     0.37\textsuperscript{ } &      0.36\textsuperscript{*} \\
        10  &         0.51 &        0.53 &    0.03\textsuperscript{*} &     0.02\textsuperscript{*} &      0.47\textsuperscript{ } \\
        20  &         0.53 &        0.54 &   -0.02\textsuperscript{*} &    -0.02\textsuperscript{*} &      0.52\textsuperscript{ } \\
        30  &         0.49 &        0.46 &    0.11\textsuperscript{*} &     0.13\textsuperscript{*} &      0.35\textsuperscript{*} \\
        \bottomrule
    \end{tabular}
\caption{Impact of training on correlation with a ViT backbone.}
\label{tab:training_vit}
\end{table}

\begin{table}[t]
    \centering
    \begin{tabular}{lrrrrrrrrrr}
        \toprule
        {} &      $\logme$ & $\totlogme$ & $\mathcal{H}$ & $\mathcal{H}_{\alpha}$ & $\transrate$ \\
        Epochs &                 &             &         &          &           \\
        \midrule
        5   &         0.46 &        0.60\textsuperscript{ } &    0.12\textsuperscript{*} &     0.09\textsuperscript{*} &      0.56 \\
        10  &        0.28 &        0.45\textsuperscript{ } &    0.07\textsuperscript{*} &     0.04\textsuperscript{*} &      0.44 \\
        20  &         0.20 &        0.36\textsuperscript{ } &   -0.09\textsuperscript{*} &    -0.13\textsuperscript{*} &      0.41 \\
        30  &        0.11 &        0.28\textsuperscript{*}  &   -0.18\textsuperscript{*} &    -0.23\textsuperscript{*} &      0.39 \\
        \bottomrule
    \end{tabular}
\caption{Impact of training on correlation with a ResNet backbone.}
\label{tab:training_resnet}
\end{table}

\subsection{Total LogME is less sensitive to training budget}

As transferability metrics aim at measuring the relationship between features extracted from a pre-trained model and target labels,
this relationship should weaken as the model is trained. To study the impact of training on the TMs we compute the correlation for 
model checkpoints at different steps in training for \textit{Real 2} datasets. We do this experiment with both a ResNet backbone and a Vision Transformer backbone. 

In tables~\ref{tab:training_vit} and \ref{tab:training_resnet}, we observe that compared to H-Score or TransRate, LogME and its extension are less sensitive to training 
with a slower and steadier decrease.  We also observe that the Faster-RCNN using a ViT backbone is more robust than the traditional architecture.

\begin{figure*}[ht!]
    \centering
    \begin{subfigure}[b]{0.24\textwidth}
        \captionsetup{justification=centering}
        \centering
        \includegraphics[width=\textwidth]{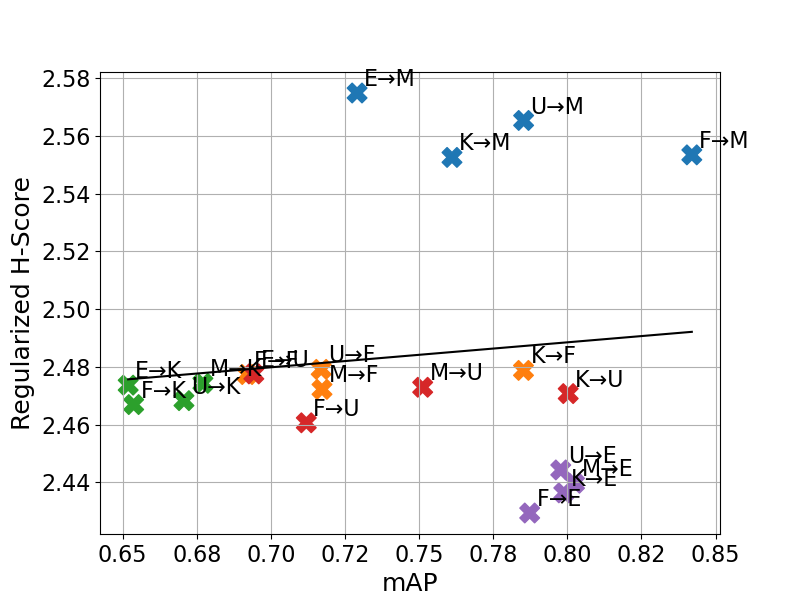}
        \caption{$\mathcal{H}_{\alpha}$ on global features  \\ $R=0.12^\ast$ \\ $R_s=-0.11^\ast$ \\ $\tau=0.01^\ast$ }
        \label{fig:synth_glob_hscore}
    \end{subfigure}
    \begin{subfigure}[b]{0.24\textwidth}
        \captionsetup{justification=centering}
        \centering
        \includegraphics[width=\textwidth]{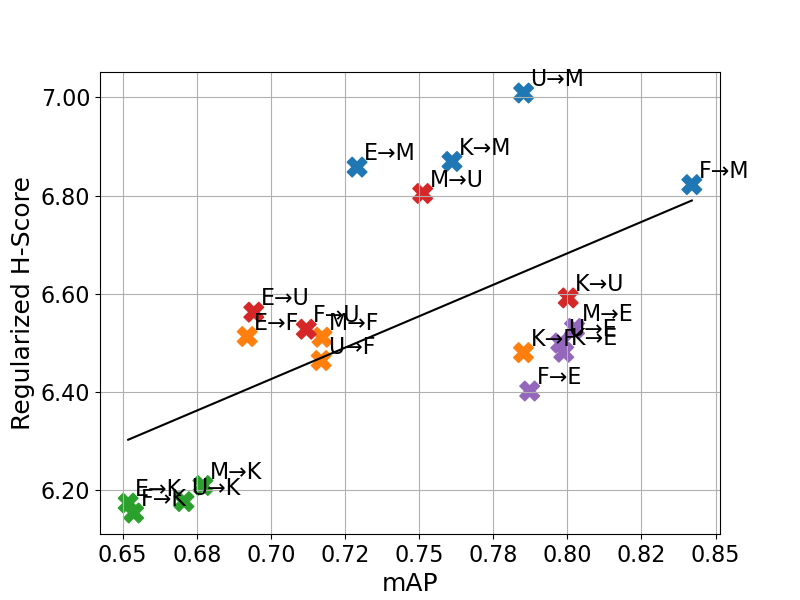}
        \caption{$\mathcal{H}_{\alpha}$ on local features\\ $R=0.60$ \\ $R_s=0.53$ \\ $\tau=0.41$ }
        \label{fig:synth_loc_hscore}
    \end{subfigure}
    \begin{subfigure}[b]{0.24\textwidth}
        \captionsetup{justification=centering}
        \centering
        \includegraphics[width=\textwidth]{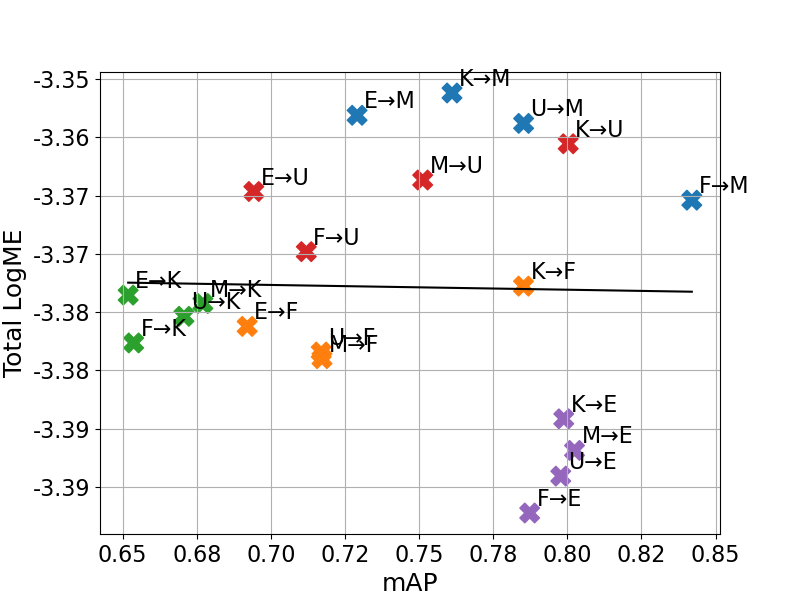}
        \caption{$\totlogme$ on global features \\ $R=-0.02^\ast$ \\ $R_s=-0.04^\ast$ \\ $\tau=-0.02^\ast$}
        \label{fig:synth_glob_tlogme}
    \end{subfigure}
    \begin{subfigure}[b]{0.24\textwidth}
        \captionsetup{justification=centering}
        \centering
        \includegraphics[width=\textwidth]{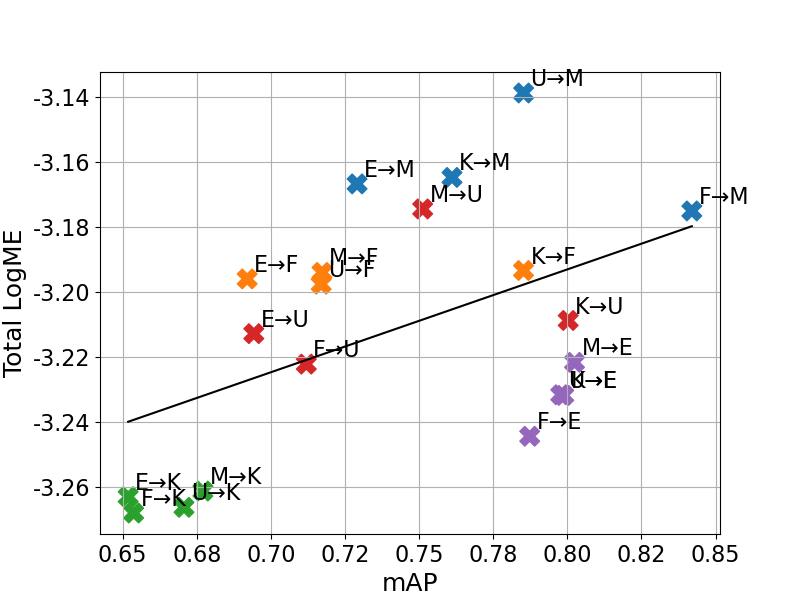}
        \caption{$\totlogme$ on local features   \\ $R=0.48$ \\ $R_s=0.41$ \\ $\tau=0.34^\ast$}
        \label{fig:synth_loc_tlogme}
    \end{subfigure}
    \caption{Overall correlation on all $M=20$ transfers for synthetic datasets. Points are colored according to the target dataset. $R$, $R_s$ and $\tau$ are respectively Pearson, Spearman and Kendall correlations.}
    \label{fig:correlation_agg}
\end{figure*}

\begin{figure*}[ht!]
    \centering
    \begin{subfigure}[b]{0.19\textwidth}
        \centering
        \captionsetup{justification=centering}
        \includegraphics[width=\textwidth]{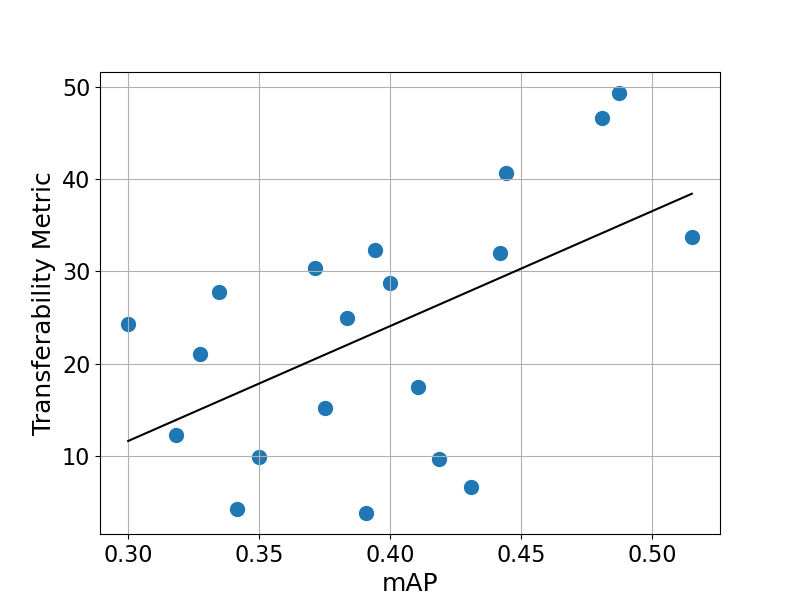}
        \caption{Best Global Metric \\ $R=0.54$ \\ $R_s=0.51$ \\ $\tau=0.36$ }
        \label{fig:baseline}
    \end{subfigure}
    \begin{subfigure}[b]{0.19\textwidth}
        \centering
        \captionsetup{justification=centering}
        \includegraphics[width=\textwidth]{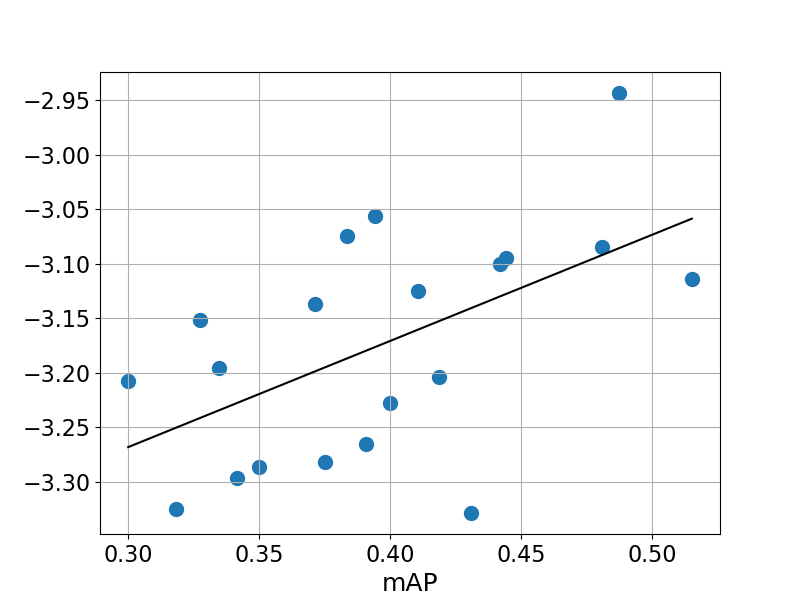}
        \caption{$\totlogme$ \\ $R=0.56$ \\ $R_s=0.52$ \\ $\tau=0.40$ }
        \label{fig:tlogme}
    \end{subfigure}
    \begin{subfigure}[b]{0.19\textwidth}
        \centering
        \captionsetup{justification=centering}
        \includegraphics[width=\textwidth]{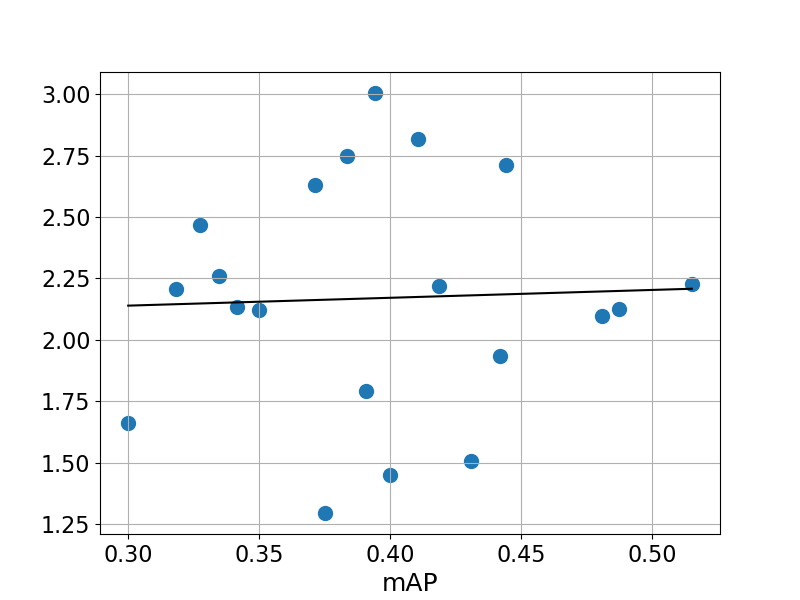}
        \caption{$\mathcal{H}$ \\ $R=0.04^\ast$ \\ $R_s=0.00^\ast$, \\ $\tau=0.03^\ast$ }
        \label{fig:hscore}
    \end{subfigure}
    \begin{subfigure}[b]{0.19\textwidth}
        \centering
        \captionsetup{justification=centering}
        \includegraphics[width=\textwidth]{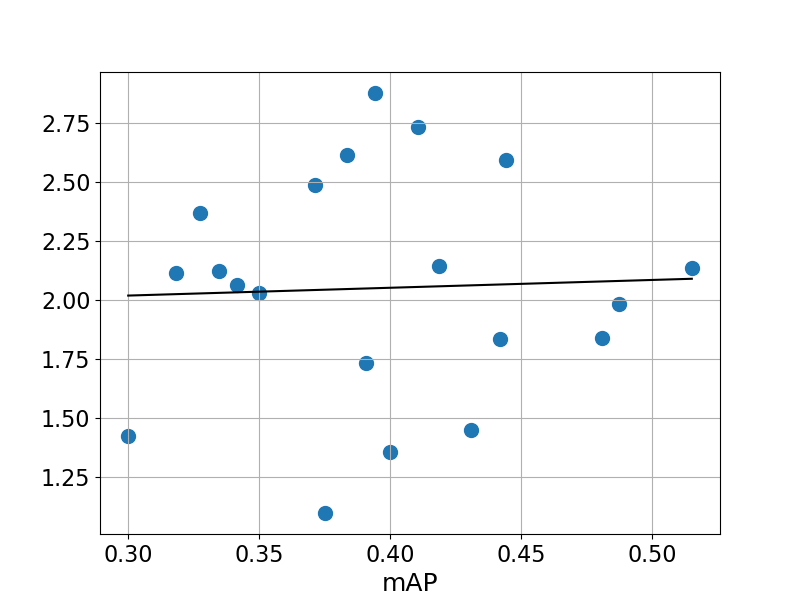}
        \caption{$\mathcal{H}_{\alpha}$ \\ $R=0.04^\ast$ \\ $R_s=0.03^\ast$ \\ $\tau=0.04^\ast$ }
        \label{fig:reg_hscore}
    \end{subfigure}
    \begin{subfigure}[b]{0.19\textwidth}
        \centering
        \captionsetup{justification=centering}
        \includegraphics[width=\textwidth]{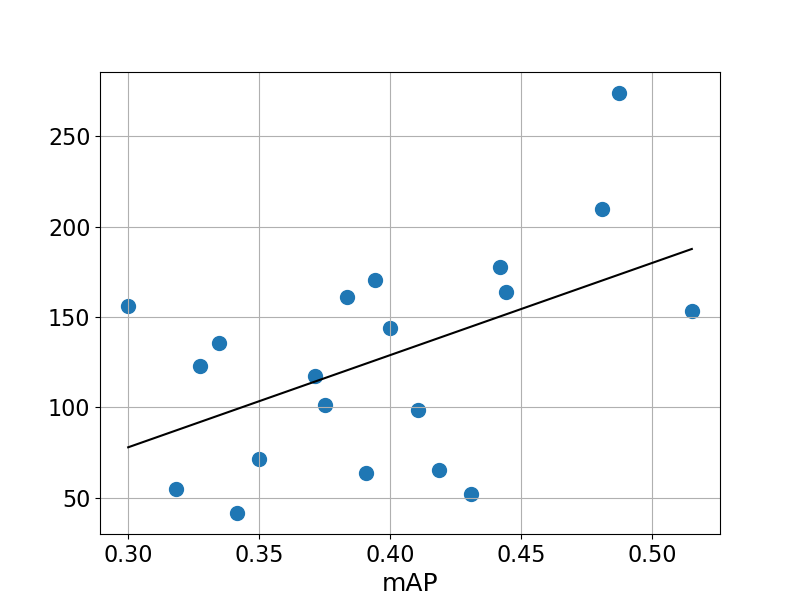}
        \caption{$\transrate$ \\ $R=0.50$ \\ $R_s=0.45$ \\ $\tau=0.29$ }
        \label{fig:transrate}
    \end{subfigure}
    \caption{Correlations between transferability metrics and mAP on \textit{Real 2} task with ViT-Faster-RCNN. Except for (a), all metrics are computed on local features. $\logme$ is omitted because having similar behavior to $\totlogme$.}
    \label{fig:comparison_corr}
\end{figure*}

\subsection{The local feature extractor method matters}
\label{subsec:layer}

Table \ref{tab:layer_oi_vit} shows the performance of the TMs when applied to features extracted from different layers.
We can observe that in all cases TMs have better behavior when using ROI-Align on FPN layers, or Multi-Scale ROI-Align, instead of using regular backbone layers, 
which would be the most natural choice.
Another important observation is that the local features from the penultimate layer $f^{FC}_{-1}$ are the least representative of the transferability measure.
This result is interesting because it confirms that the standard approach used to measure transferability in classification is less pertinent in OD.

\subsection{Reproducibility}\label{subsec:reproducibility}

The code used to generate synthetic datasets, prepare real datasets, train and evaluate models, extract global and local features, 
compute TMs, and evaluate their performance is available at this GitHub repository\footnote{\url{https://github.com/dataiku-research/transferability_metrics_for_object_detection}}.
More details on the experimental setup and datasets can be found in the Supplementary Materials.

\section{Conclusions}\label{sec:conclusion}

In this paper we presented the first extension of 
state-of-the-art transferability metrics to the object-detection problem.
We studied the generalization ability of Faster-RCNN object-detectors, with ResNet and Transformer backbone,
and proposed a method to extract object-level features to plug into existing
transferability metrics, initially designed for classification.
This approach is more relevant for object detection than using image-level representations 
from backbone layers and our study highlights the importance of object-level feature extraction 
for the quality of transferability metrics.

\begin{table}[t]
    \centering
    \begin{tabular}{lrrrrr}
        \toprule
        {} &$\logme$ & $\totlogme$& $\mathcal{H}$ & $\mathcal{H}_{\alpha}$ & $\transrate$ \\
        Layer       &       &             &         &          &           \\
        \midrule
        $f_5$     &  0.52 &        0.53 &   -0.08\textsuperscript{*} &     0.01\textsuperscript{*} &      0.45 \\
        $f^{MS}$      &  0.56 &        0.55 &    0.04\textsuperscript{*} &     0.04\textsuperscript{*} &      0.50 \\
        $f_{FPN, 0}$      &  0.58 &        0.61 &    0.00\textsuperscript{*} &     0.04\textsuperscript{*} &      0.57 \\
        $f_{FPN, 1}$       &  0.58 &        0.61 &   -0.00\textsuperscript{*} &     0.06\textsuperscript{*} &      0.58 \\
        $f_{FPN, 2}$       &  0.58 &        0.61 &    0.02\textsuperscript{*} &     0.08\textsuperscript{*} &      0.58 \\
        $f_{FPN, 3}$        &  0.57 &        0.62 &    0.05\textsuperscript{*} &     0.11\textsuperscript{*} &      0.58 \\
        $f^{FC}_{-1}$  &  0.51 &        0.48 &    0.08\textsuperscript{*} &     0.08\textsuperscript{*} &      0.36 \\
        \bottomrule
        \end{tabular}
\caption{Impact of layer choice on correlation.}
\label{tab:layer_oi_vit}
\end{table}

We also introduced the $\totlogme$ transferability metric, an extension of $\logme$~\cite{you2021logme} 
taking into account both object class and location prediction tasks. 
In our experiments, comparing several 
transferability metrics on real and synthetic object-detection transfer scenarios, our approach 
applying $\totlogme$ to object-level features is found to be a robust solution 
correlating with transfer performance, especially with a constrained 
training budget.

Exploring the extension of transferability metrics for other architectures, 
in particular testing our method with single-stage detectors, for which the proposed feature extraction is easily applicable, and including the region-proposal performance in the metric, will be the focus of our future work.

\newpage
\bibliographystyle{named}
\bibliography{main}

\newpage
\onecolumn
\section{Supplementary materials}\label{sec:supp_mat}
\subsection{Object-Detection Datasets}
\subsubsection{Synthetic Data} 
Detailed code for generating synthetic data can be found in our repository. MNIST, KMNIST, FMNIST, and USPS have 10 classes each and EMNIST has 26 classes.
For each dataset, we have a repartition of 1000/200 train/test samples. Images are of size 224x224.
\subsubsection{Real Data} 
\begin{enumerate}
    \item BCCD\footnote{\url{https://www.tensorflow.org/datasets/catalog/bccd}} \cite{BCCD_Dataset}:  The dataset has 364 color images of size 480x660 with a 277/87 train/test split and 10 classes.
    \item CHESS\footnote{\url{https://public.roboflow.com/object-detection/chess-full}}  : The dataset has 693 color images of size 416x416 and 12 classes. We sample 4/5 of the images for training and the rest for testing.
    \item Global Wheat\footnote{\url{https://www.kaggle.com/c/global-wheat-detection}} \cite{david2020global} : The dataset has 3373 color images of size 1024*1024 and only one class. We sample 4/5 of the images for training and the rest for testing.
    \item VOC\footnote{\url{http://host.robots.ox.ac.uk/pascal/VOC/}} \cite{Everingham10}: The dataset has 5717/5823 train/test images of multiple sizes with 20 different classes.
    \item Vehicle\footnote{\url{https://storage.googleapis.com/openimages/web/index.html}} : The dataset is a subset of OpenImages \cite{kuznetsova2020open} with 13 classes of vehicles only and 166k images of mutliple sizes. The code to download this dataset is available in our repository. 
    We sample 10000 images for training and 2000 for testing.
    \item Bootstrapped Datasets: The code to create the bootstrapped datasets from OpenImages is in our repository. Each dataset has a repartition of 1000/200 train/test samples and 5 classes.
\end{enumerate}
\subsection{General details}
All our models and training code are implemented in PyTorch \cite{paszke2019pytorch}. Both architectures of Faster-RCNN with ResNet and ViT backbones come from torchvision \cite{marcel2010torchvision}. We also use weights from torchvision
for ResNet backbone trained on ImageNet and for Faster-RCNN trained on the COCO dataset. During the transfer, we only retrain the head of the architecture while freezing the backbones. 
However, the Feature Pyramid network is always retrained even if part of the backbone in the PyTorch implementation.
During training, all images are resized to the size of (800x800) which is the minimal size in the PyTorch implementation of Faster-RCNN. We use reference scripts from torchvision to achieve training of the models.
We extract local features using the same parameters of ROI-Align as used in the PyTorch implementation of Faster-RCNN. Then we reduce the spatial dimension to one
with a final average pooling. Our implementations of $\transrate$ and $\mathcal{H}_{\alpha}$ have been merged in the Transfer Learning Library\footnote{\url{https://github.com/thuml/Transfer-Learning-Library}} along with their implemenations of $\logme$ and $\mathcal{H}$.
\subsection{Experiment details on synthetic datasets}
On synthetic datasets, we simulate \textit{pre-trained models} by training a Faster-RCNN
 with a backbone pre-trained on ImageNet on 30 epochs. We use a two-step fine-tuning approach by first training the head only on 10 epochs with a learning rate of 10e-4. 
 Then we train the head jointly with the backbone on 20 epochs and with a learning rate divided by 10. We use an Adam Optimizer on batches of 4 with a Reduce on Plateau callback.
 Tables \ref*{tab:source_synth_detailed} and \ref*{tab:target_synth_detailed} contain detailed results of the experiments.
\begin{table*}
    \centering
    \begin{tabular}{lc|rrrrr|rrrrr}
        
        \toprule
        {} &    &  \multicolumn{5}{c|}{Global Features} & \multicolumn{5}{c}{Local Features}\\
        \midrule
        {} &   &   $\logme$ & $\totlogme$ & $\mathcal{H}$ & $\mathcal{H}_{\alpha}$ & $\transrate$ & LogME & TLogME & $\mathcal{H}$ & $\mathcal{H}_{\alpha}$ & $\transrate$ \\
         $\mathcal{T}_T$ &  Corr &         &             &         &          &           &       &             &         &          &           \\
        \midrule
        \multirow{3}{*}{MNIST} & $R$ &      -0.39 &       -0.80 &    \textbf{0.99} &    -0.38 &     -0.53 & -0.18 &       -0.13 &    0.95 &    -0.15 &     -0.20 \\
        & $R_s$ &    -0.40 &        -0.40 &     1.00 &      0.00 &      -0.80 &  -0.40 &        -0.40 &     0.80 &     -0.40 &      -0.40 \\
        & $\tau$& -0.33 &       -0.33 &    1.00 &     0.00 &     -0.67 & -0.33 &       -0.33 &    0.67 &    -0.33 &     -0.33 \\
        \midrule
        \multirow{3}{*}{KMNIST} & $R$  &       \textbf{0.98} &        0.13 &    0.15 &    0.93 &     -0.89 &  0.60 &        0.90 &   -0.47 &     0.73 &      0.86 \\
        & $R_s$&   0.80 &        -0.40 &     0.40 &      0.60 &      -1.00 &   0.60 &         0.60 &     0.20 &     0.80 &      0.89 \\
        & $\tau$& 0.67 &       -0.33 &    0.33 &     0.33 &     -1.00 &  0.33 &        0.33 &    0.00 &     0.67 &      0.67 \\
        \midrule
        \multirow{3}{*}{EMNIST} & $R$   &       0.29 &        0.74 &    0.73 &     0.47 &     -0.68 &  \textbf{1.00} &        0.84 &    0.96 &     0.96 &      0.59 \\
        & $R_s$&    0.20 &         0.80 &     0.80 &      0.00 &      -0.40 &   1.00 &         0.8 &     1.00 &      1.00 &       0.40 \\
        & $\tau$& 0.00 &        0.67 &    0.67 &     0.00 &     -0.33 &  1.00 &        0.67 &    1.00 &     1.00 &      0.33 \\
        \midrule
        \multirow{3}{*}{FMNIST} & $R$ &      -0.27 &        0.76 &    0.19 &     \textbf{0.78} &     -0.11 &  0.77 &        0.53 &    0.34 &     0.59 &     -0.40 \\
        & $R_s$&    0.00 &         0.20 &     0.20 &      0.40 &      -0.40 &   0.80 &         0.80 &     0.40 &      0.80 &      -0.40 \\
        & $\tau$&   0.00 &        0.00 &    0.00 &     0.33 &     -0.33 &  0.67 &        0.67 &    0.33 &     0.67 &     -0.33 \\
        \midrule
        \multirow{3}{*}{USPS} & $R$ &       0.26 &        \textbf{0.74} &    0.05 &     0.12 &     -0.32 &  0.31 &        0.33 &    0.01 &     0.30 &      0.22 \\
        & $R_s$&  0.60 &         0.80 &     0.20 &      0.00 &      -0.40 &   0.60 &         0.60 &     0.40 &      0.60 &       0.00 \\
        & $\tau$&  0.33 &        0.67 &    0.00 &     0.00 &     -0.33 &  0.33 &        0.33 &    0.33 &     0.33 &      0.00 \\
        \midrule
        \midrule
        \multirow{3}{*}{\textbf{Mean}} & $R$  &  0.17 & 0.31  & 0.42 & 0.38& -0.50 &   \textbf{0.50} &       \textbf{0.50} &    0.36 &     0.49 &      0.21 \\
        & $R_s$& 0.24 &        0.20 &     0.52 &     0.20 &      -0.60 &  0.52 &        0.48 &     0.56 &     0.56 &      0.08 \\
        & $\tau$& 0.13 &        0.14 &     0.40 &     0.13 &      -0.53 &  0.40 &        0.33 &     0.47 &     0.47 &      0.07 \\
        \bottomrule
        \end{tabular}
    \caption{Source task selection on synthetic datasets}
    \label{tab:source_synth_detailed}
\end{table*}
\begin{table*}
    \centering
    \begin{tabular}{lc|rrrrr|rrrrr}
        \toprule
        {} &    &  \multicolumn{5}{c|}{Global Features} & \multicolumn{5}{c}{Local Features}\\
        \midrule
        {} &   &   $\logme$ & $\totlogme$ & $\mathcal{H}$ & $\mathcal{H}_{\alpha}$ & $\transrate$ & LogME & TLogME & $\mathcal{H}$ & $\mathcal{H}_{\alpha}$ & $\transrate$ \\
         $\mathcal{T}_S$ &  Corr &         &             &         &          &           &       &             &         &          &           \\
        \midrule
        \multirow{3}{*}{MNIST} & $R$ &       -0.62 &       -0.39 &   -0.41 &     0.44 &     -0.53 &  0.22 &        0.30 &   -0.63 &    \textbf{0.52} &     -0.27 \\
        & $R_s$ &    -0.40 &        -0.40 &    -0.40 &      0.40 &      -0.80 &   \textbf{0.40} &         \textbf{0.40} &    -0.40 &      \textbf{0.40} &      -0.40 \\
        & $\tau$& -0.33 &       -0.33 &   -0.33 &     0.33 &     -0.67 &  \textbf{0.33} &        \textbf{0.33} &   -0.33 &     \textbf{0.33} &     -0.33 \\
        \midrule
        \multirow{3}{*}{KMNIST} & $R$  &     -0.57 &       -0.57 &   -0.61 &    -0.93 &     -0.02 & -0.83 &       -0.90 &   -0.60 &    -0.97 &     -0.82 \\
        & $R_s$ &       -0.60 &        -0.40 &    -0.40 &     -0.80 &       0.00 &  -0.80 &        -0.80 &    -0.40 &     -0.80 &      -0.80 \\
        & $\tau$ & -0.33 &       -0.33 &   -0.33 &    -0.67 &      0.00 & -0.67 &       -0.67 &   -0.33 &    -0.67 &     -0.67 \\
        \midrule
        \multirow{3}{*}{EMNIST} & $R$   &     0.88 &        0.71 &    0.68 &     0.99 &     -0.28 &  0.97 &        \textbf{1.00} &    0.49 &     0.98 &      0.86 \\
        & $R_s$ &    0.80 &         0.80 &     0.80 &      0.80 &      -0.80 &   0.80 &         0.80 &     0.80 &      0.80&      0.80\\
        & $\tau$ &  0.67 &        0.67 &    0.67 &     0.67 &     -0.67 &  0.67 &        0.67 &    0.67 &     0.67 &      0.67 \\
        \midrule
        \multirow{3}{*}{FMNIST} & $R$ &    0.04 &        0.09 &    0.17 &     0.89 &     -0.75 &  0.77 &        0.81 &   -0.16 &     \textbf{0.91} &      0.61 \\
        & $R_s$ &    0.40 &         0.40 &     0.40 &      0.80 &      -0.40 &   0.80 &         0.80 &     0.40 &      0.80 &       0.4 \\
        & $\tau$ &    0.33 &        0.33 &    0.33 &     0.67 &     -0.33 &  0.67 &        0.67 &    0.33 &     0.67 &      0.33 \\
        \midrule
        \multirow{3}{*}{USPS} & $R$ &      -0.29 &       -0.02 &   -0.02 &     \textbf{0.65} &     -0.61 &  0.46 &        0.52 &   -0.20 &     \textbf{0.65} &      0.20 \\
        & $R_s$ &   -0.20 &        -0.40 &    -0.40 &      0.40 &      -0.60 &   0.40 &         0.40 &    -0.40 &      0.40 &      -0.20 \\
        & $\tau$ &  0.00 &       -0.33 &   -0.33 &     0.33 &     -0.33 &  0.33 &        0.33 &   -0.33 &     0.33 &      0.00 \\
        \midrule
        \midrule
        \multirow{3}{*}{\textbf{Mean}} & $R$  &  -0.11 & -0.04 & -0.04 &  0.41 & -0.43 &   0.32 &        0.35 &    -0.22 &     \textbf{0.42} &      0.12 \\
        & $R_s$ &          0.00 &     0.00 &     0.00 &      0.32 &  -0.52 &        0.32 &     0.32 &    0.00 &      0.32  & -0.04\\
        & $\tau$ & 0.07 &        0.00 &     0.00 &     0.27 &      -0.40 &  0.27 &   0.27 &    0.00 &     0.27 &     0.00 \\
        \bottomrule
        \end{tabular}
    \caption{Target task selection on synthetic datasets}
    \label{tab:target_synth_detailed}
\end{table*}
\subsection{Experiment details on real datasets}
On real datasets we use a Faster-RCNN pre-trained on COCO and train its head only on 30 epochs on the target datasets. We once again use an Adam Optimizer with a learning rate of 10e-4 and Reduce on Plateau callback.
For ResNet backbone, batches are of size 8 and for ViT backbone, batches are of size 6 because of the increased size of the model.  Table \ref*{tab:real_detailed} contains the detailed results of these experiments.
\begin{table*}[h!]
    \centering
    \begin{tabular}{lcc|rrrrr|rrrrr}
        \toprule
        {} &  &    &  \multicolumn{5}{c|}{Global Features} & \multicolumn{5}{c}{Local Features}\\
        \midrule
        {} &   &   &  $\logme$ & $\totlogme$ & $\mathcal{H}$ & $\mathcal{H}_{\alpha}$ & $\transrate$ & LogME & TLogME & $\mathcal{H}$ & $\mathcal{H}_{\alpha}$ & $\transrate$ \\
         $\mathcal{T}_S$ & Backbone &  Corr &         &             &         &          &           &       &             &         &          &           \\
        \midrule
        \multirow{3}{*}{Real 1} & \multirow{3}{*}{ResNET} & $R$ &   -0.00 &       -0.01 &   -0.58 &    -0.54 &     -0.82 &  0.10 &        \textbf{0.47} &   -0.01 &    -0.03 &     -0.68 \\
        & & $R_s$ &   -0.60 &        -0.60 &    -0.50 &     -0.50 &      -0.60 &  -0.20 &         0.70 &    -0.10 &     -0.10 &      -0.60 \\
        &  &$\tau$&  -0.40 &        -0.40 &    -0.20 &     -0.20 &      -0.40 &  -0.20 &         0.60 &     0.00 &      0.00 &      -0.40 \\
        \midrule
        \multirow{3}{*}{Real 2} & \multirow{3}{*}{ResNET} & $R$ &  0.28 &        0.33 &    0.14 &     0.14 &      0.20 &  0.15 &        0.31 &   -0.20 &    -0.24 &      \textbf{0.43} \\
        & & $R_s$ &    0.13 &        0.23 &    0.07 &     0.08 &      0.09 &  0.04 &        0.22 &   -0.22 &    -0.26 &      0.31 \\
        &  &$\tau$& 0.08 &        0.16 &    0.03 &     0.05 &      0.08 &  0.04 &        0.17 &   -0.16 &    -0.21 &      0.21 \\
        \midrule
        \multirow{3}{*}{Real 2} & \multirow{3}{*}{ViT} & $R$ & 0.53 &        0.52 &    0.13 &     0.22 &      0.54 &  \textbf{0.56} &        \textbf{0.56} &    0.04 &     0.04 &      0.50 \\
        & & $R_s$ &     0.51 &        0.51 &    0.06 &     0.18 &      0.51 &  0.52 &        0.52 &   -0.00 &     0.03 &      0.45 \\
        &  &$\tau$& 0.36 &        0.35 &    0.03 &     0.14 &      0.36 &  0.36 &        0.40 &    0.03 &     0.04 &      0.29 \\
        \midrule
         \bottomrule
        \end{tabular}
    \caption{Target task selection on real datasets}
    \label{tab:real_detailed}
\end{table*}

\end{document}